\documentclass{article}

\usepackage{microtype}
\usepackage{graphicx}
\usepackage{subcaption}
\usepackage{booktabs} 
\usepackage{enumitem}
\usepackage{pifont}
\usepackage{threeparttable}
\usepackage{booktabs}

\usepackage{hyperref}

\usepackage[preprint]{icml2026}

\usepackage{amsmath}
\usepackage{amssymb}
\usepackage{mathtools}
\usepackage{amsthm}
\usepackage{multirow}

\usepackage[capitalize,noabbrev]{cleveref}

\theoremstyle{plain}
\newtheorem{theorem}{Theorem}[section]

\newtheorem{lemma}[theorem]{Lemma}

\theoremstyle{definition}

\newtheorem{assumption}[theorem]{Assumption}
\theoremstyle{remark}

\newcommand{\tcut}{t_{\text{cutoff}}}

\usepackage[textsize=tiny]{todonotes}

\icmltitlerunning{
Efficient and Stable Reinforcement Learning for Diffusion Language Models}

\begin{document}

\twocolumn[
  \icmltitle{
  Efficient and Stable Reinforcement Learning for Diffusion Language Models}



  \icmlsetsymbol{equal}{*}

  \begin{icmlauthorlist}
    \icmlauthor{Jiawei Liu}{ustc,cityu}
    \icmlauthor{Xiting Wang}{ruc}
    \icmlauthor{Yuanyuan Zhong}{pku}
    \icmlauthor{Defu Lian}{ustc}
    \icmlauthor{Yu Yang}{cityu}
  \end{icmlauthorlist}

  \icmlaffiliation{ustc}{State Key Laboratory of Cognitive Intelligence, University of Science and Technology of China, Hefei, China}
  \icmlaffiliation{cityu}{City University of Hong Kong, Hong Kong, China}
  \icmlaffiliation{ruc}{Gaoling School of Artificial Intelligence, Renmin University of China, Beijing, China}
  \icmlaffiliation{pku}{School of Software and Microelectronics, Peking University, Beijing, China}

  \icmlcorrespondingauthor{Xiting Wang}{xitingwang@ruc.edu.cn}
  \icmlcorrespondingauthor{Defu Lian}{liandefu@ustc.edu.cn}
  \icmlcorrespondingauthor{Yu Yang}{yuyang@cityu.edu.hk}

  \icmlkeywords{Reinforcement Learning, Machine Learning, ICML}

  \vskip 0.3in
]



\printAffiliationsAndNotice{}  
 
\begin{abstract}
  Reinforcement Learning (RL) is crucial for unlocking the complex reasoning capabilities of Diffusion-based Large Language Models (dLLMs). However, applying RL to dLLMs faces unique challenges in efficiency and stability. To address these challenges, we propose Spatio-Temporal Pruning (STP), a framework designed to simultaneously improve the efficiency and stability of RL for dLLMs. STP compresses the redundancy in the generative process through: (1) \textit{spatial pruning}, which constrains the exploration space using static priors; and (2) \textit{temporal pruning}, which bypasses redundant late-stage refinement steps. Our theoretical analysis demonstrates that STP strictly reduces the variance of the log-likelihood estimation, thereby ensuring more stable policy updates. Extensive experiments demonstrate that STP surpasses state-of-the-art baselines in both efficiency and accuracy.  Our code is available at \url{https://github.com/Lolo1222/STP}.
\end{abstract}

\section{Introduction}
\label{sec:intro}
\newcommand{\cmark}{\ding{51}} 
\newcommand{\xmark}{\ding{55}} 

\begin{table}[t!]
\centering
\begin{threeparttable}
\caption{Comparison of RL algorithms for dLLMs. STP (Ours) achieves the best efficiency and accuracy empirically with a unique theoretical guarantee on stability. Symbols \textbf{\xmark}, \textbf{\cmark}, and \textbf{\cmark\cmark} denote \textit{limited}, \textit{good}, and \textit{state-of-the-art} performance, respectively.
}
\label{tab:comparison}
\begin{small}
\begin{tabular*}{\columnwidth}{@{\extracolsep{\fill}} lccc}
\toprule
Method & Efficiency & Stability & Accuracy \\
\midrule
Standard GRPO       & \xmark            & \cmark            & \cmark            \\
Diffu-GRPO          & \cmark            & \xmark            & \xmark            \\
UniGRPO             & \cmark            & \xmark            & \xmark            \\
\midrule
\textbf{STP (Ours)} & \textbf{\cmark\cmark}\tnote{1} & \textbf{\cmark\cmark}\tnote{2} & \textbf{\cmark\cmark}\tnote{3} \\
\bottomrule
\end{tabular*}
\end{small}

\begin{tablenotes}
    \footnotesize
    \item[1] \textbf{+13.1\%} training speedup over the fastest baseline. 
    \item[2] Theoretically provable variance reduction (see Section~\ref{sec:theory}).
    \item[3] Up to \textbf{81.7\%} relative improvement on logic reasoning tasks.
\end{tablenotes}
\end{threeparttable}
\end{table}

\begin{figure*}[t]
  \begin{center}
    \centerline{\includegraphics[width=\textwidth]{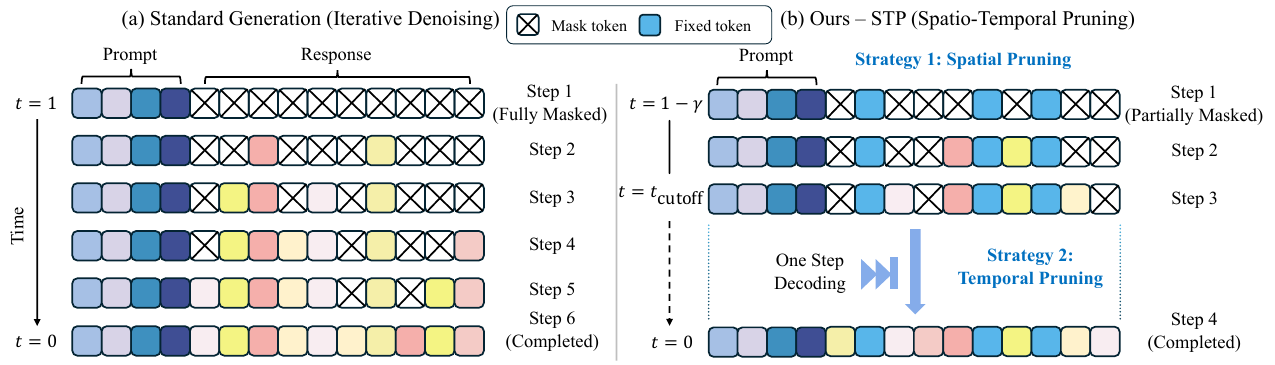}}
    \caption{
Comparing (a) standard sequence generation with iterative denoising with (b) our generation process with spatio-temporal purning, which reduces computational cost and ELBO variance theoretically while improving empirical accuracy. 
    }
    \label{fig1}
  \end{center}
\end{figure*}

Diffusion-based large language models (dLLMs) have emerged as an alternative paradigm for generative modeling~\cite{nie2025large, ouyour, yang2025mmada}. By iteratively refining sequences through a denoising process, dLLMs offer a non-autoregressive alternative that breaks the left-to-right sequential generation constraint inherent to traditional autoregressive (AR) models~\cite{lou2024discrete}.
Open-weight dLLMs like LLaDA~\cite{nie2025large} now demonstrate performance competitive with AR models of comparable scale, while closed-source models such as Mercury~\cite{khanna2025mercury} and Gemini Diffusion~\cite{gemini_diff} achieve more than an order-of-magnitude generation speed-up over AR models.


Applying reinforcement learning (RL) to dLLMs is essential for improving their reasoning abilities and ensuring safety~\cite{guo2025deepseek,wang2024comprehensive}.
However, the integration of RL with diffusion-based language models differs fundamentally from the autoregressive (AR) setting. In AR models, token probabilities are factorized along a single left-to-right generation order, enabling exact log-likelihood computation via the chain rule ~\cite{radford2018improving}. 
In contrast, diffusion language models generate sequences with bidirectional context, where tokens can be updated in multiple orders. This makes the exact likelihood computation -- which marginalizes over a combinatorial set of generation trajectories -- intractable and typically approximated via multi-step Monte Carlo sampling~\cite{zhao2025d1,zhu2025llada}. This reliance on estimation has led to two challenges for RL training:
\begin{itemize}[nosep,leftmargin=1em,labelwidth=*,align=left]
\item \textbf{C1. Efficiency.} Likelihood estimation relies on multi-step Monte Carlo sampling over diffusion timesteps, increasing the computational cost of trajectory generation.
\item  \textbf{C2. Stability.} Stochastic estimation over both diffusion time and token states introduces variance into the policy gradient estimator, which destabilizes the RL training process~\cite{zhao2025d1, wang2025spg}. 
\end{itemize}

As summarized in Table~\ref{tab:comparison}, existing RL methods for dLLMs fail to simultaneously achieve high computational efficiency and training stability. Standard GRPO is computationally expensive due to multi-step Evidence Lower Bound (ELBO) estimation. To mitigate this burden, efficiency-oriented approaches like Diffu-GRPO~\cite{zhao2025d1} and UniGRPO~\cite{yang2025mmada} employ a one-step denoising strategy to approximate token likelihoods. Although these methods improve efficiency, they introduce estimation variance, which leads to unstable gradient updates and suboptimal accuracy (see experimental results in Table~\ref{tab:main_results}). 
We propose an alternate method that improves the efficiency by reducing sampling redundancy across \textit{Spatial} (sequence length) and \textit{Temporal} (diffusion steps) dimensions. As shown in Table~\ref{tab:comparison}, our method achieves the best efficiency and has a theoretically guaranteed variance reduction, leading to better stability and state-of-the-art accuracy. Moreover, our method is complementary to current RL improvements and can be combined with them to achieve greater improvements (see experimental results in Table~\ref{tab:orthogonality}).
More specifically, our framework, \textbf{Spatio-Temporal Pruning (STP)}, consists of two pruning strategies as shown in Figure~\ref{fig1}:
(1) \textit{spatial pruning}, 
which incorporates the \textit{semi-offline} strategy~\cite{chen2023semi} by mixing offline data with online generation,
thereby constraining the search space to prevent inefficient exploration; and
(2) \textit{temporal pruning}, 
which leverages the non-autoregressive nature of dLLMs to compress redundant refinement steps in the late stage.
This is based on the insight that later steps are primarily dedicated to local syntactic refinement~\cite{huang2025reinforcing}, allowing for redundancy reduction without compromising generation quality.

Our main contributions are summarized as follows:
\begin{itemize}[nosep,leftmargin=1em,labelwidth=*,align=left]
    \item We propose {Spatio-Temporal Pruning (STP)}, which is, to the best of our knowledge, the first work to simultaneously optimize the computational efficiency and training stability of RL training for dLLMs.
    \item We provide a theoretical analysis demonstrating that STP not only reduces computational cost of trajectory sampling but also lowers the variance of the log-likelihood estimation, thereby enabling more stable and efficient reinforcement learning for dLLMs.
    \item 
    Extensive experiments demonstrate the effectiveness of STP: it reduces training time by 13.1\% vs. the fastest baseline (\textit{Diffu-GRPO}) while consistently outperforming the strongest baseline (\textit{GRPO w/ ELBO}), achieving steady gains on math benchmarks and up to 81.7\% relative improvement on logic benchmarks.
\end{itemize}

\section{Preliminaries}
\label{sec:preliminaries}
\subsection{GRPO for dLLMs}
Although our framework is suitable for various RL algorithms, we adopt Group Relative Policy Optimization (GRPO)~\cite{shao2024deepseekmath,guo2025deepseek} as our base RL algorithm for its memory efficiency and its prevalence in RL for dLLMs~\cite{zhao2025d1, yang2025mmada}.
\paragraph{Objective Formulation.}
GRPO samples a group of $G$ outputs $\{o_1, \dots, o_G\}$ from the old policy $\pi_{\theta_{\text{old}}}$ for each prompt. For each output $o_i$, a group-relative advantage is computed as $\hat{A}_i = (R(o_i) - \mu_R) / \sigma_R$, where $R$ is the reward function, $\mu_R$ and $\sigma_R$ are the mean and standard deviation of rewards within the group. The objective maximizes:
\begin{align}
\label{eq:grpo}
    & \mathcal{J}_{\mathrm{GRPO}}(\theta)  = \mathbb{E}_{\substack{ \{o_i\}_{i=1}^G \sim \pi_\text{old}}}
    \Bigg[ \left( \frac{1}{G} \sum_{i=1}^G \frac{1}{\left|o_i\right|} \sum_{k=1}^{\left|o_i\right|} \min \Bigl( \rho_i^k A_i^k, \right. \nonumber \\
    & \quad \left. \operatorname{clip}\left(\rho_i^k, 1-\varepsilon, 1+\varepsilon\right) A_i^k \Bigr) \right) 
    - \beta D_{\mathrm{KL}} \left[ \pi_\theta \| \pi_{\mathrm{ref}} \right] \Bigg]
\end{align}
where $\rho_i^k = \frac{\pi_\theta(o_i^k | c, o_i^{<k})}{\pi_{\theta_{\text{old}}}(o_i^k | c, o_i^{<k})}$ is the importance sampling ratio for the $k$-th token of output $i$, $\epsilon$ and $\beta$ are hyper-parameters, and $\pi_{\text{ref}}$ is a fixed reference policy. 
This objective encourages the policy to increase the probability of tokens that yield high advantages $A_i^k$, while the clipping mechanism and KL penalty jointly ensure training stability by constraining the policy update.
\subsection{Masked Diffusion Large Language Models} 
Masked dLLMs~\cite{austin2021structured, shi2024simplified, nie2025large} generate text ${x_0} = (x_0^1, \dots, x_0^L)$ of length $L$ through an iterative denoising process indexed by continuous time $t \in [0, 1]$. 
The forward process $q(x_t|x_0)$ independently masks tokens in the clean sequence $x_0$. 
Following the standard linear schedule~\cite{nie2025large, sahoo2024simple}, the conditional distribution for each token at time $t$ is given by:
\begin{equation}
    q(x_t^i| x_0^i) = \begin{cases}1-t, & x_t^i=x_0^i \\ t, & x_t^i=\texttt{[MASK]}\end{cases}.
\end{equation}
Thus, $t$ equals to the proportion of masks in the sequence, $x_1$ is a fully masked sequence, and $x_0$ is the clean data. The reverse generative process recovers $x_0$ from $x_1$ by iteratively sampling $x_{s}$ from $x_t$ ($s < t$) using a mask predictor $\pi_\theta(x_0|x_t)$.

\paragraph{Intractability and ELBO.}
Unlike AR models where likelihood is factorizable, the exact log-likelihood $\log \pi_\theta(x_0)$ of Masked dLLMs requires marginalizing over all possible high-dimensional latent trajectories $x_{1:T}$, which is computationally intractable~\cite{zhao2025d1, wang2025spg}.
Therefore, optimization is performed on the Evidence Lower Bound (ELBO), which serves as a tractable surrogate. For a given sequence $x_0$, the ELBO is defined as:
\begin{equation}
    \label{eq:elbo_def}
    \begin{aligned}
    \mathcal{B}_{\pi_\theta}(x_0) 
    \triangleq \mathbb{E}_{\substack{t \sim \mathcal{U}(0,1) \\ x_t \sim q(x_t|x_0)}} 
    \bigg[ \frac{1}{t} \sum_{i =1}^L &\mathbb{I}(x_t^i=\texttt{[MASK]}) \\
    &\cdot \log \pi_\theta(x_0^i | x_t) \bigg],
    \end{aligned}
\end{equation}
In practice, we expend considerable computation to mitigate the high variance inherent in its Monte Carlo estimation~\cite{nie2025large}.
\begin{figure}[t]
  \begin{center}
    \centerline{\includegraphics[width=\columnwidth]{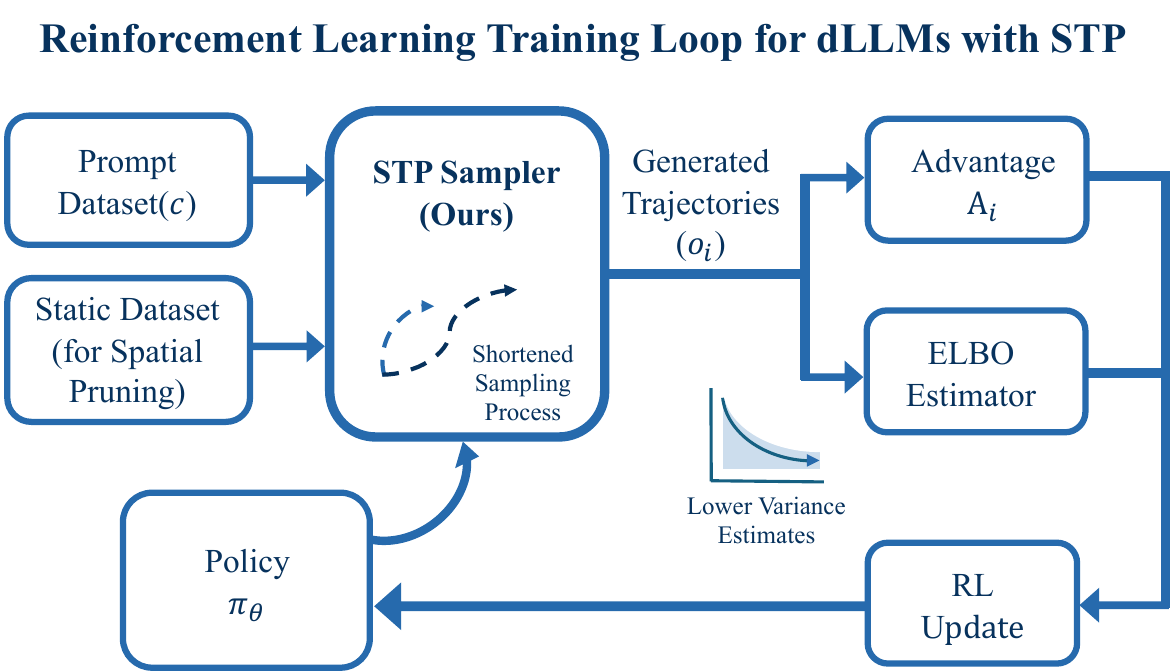}}
    \caption{
      Replacing standard sampling with STP accelerates trajectory generation for exploration. Furthermore, by constraining the sampling space, STP yields lower-variance ELBO estimates, facilitating stable policy updates.
    }
\label{fig2}
  \end{center}
\end{figure}

\paragraph{Applying GRPO to dLLMs}
For dLLMs, the final clean data $x_0$ obtained after the denoising process serves as the trajectory $o$ for the subsequent RL training.
The intractable token likelihoods in $\rho_i^k$ is approximated using the exponential of the token-wise ELBO derived from Eq.~\eqref{eq:elbo_def}:
\begin{equation}
    \rho_i^k(\theta) \approx \exp \left( \mathcal{B}_{\pi_\theta}(o_i^k) - \mathcal{B}_{\pi_{\theta_{\text{old}}}}(o_i^k) \right).
\end{equation}
As shown in Figure~\ref{fig2}, the ELBO estimator enables the application of policy gradient methods to diffusion models.
However, this approximation suffers from high variance, which destabilizes RL training.
This motivates us to modify the trajectory generation process to reduce variance by STP Sampler, leading to more stable
and efficient policy learning.
\section{Spatio-Temporal Pruning}
\label{sec:method}
The generative process of a dLLM can be formulated as a trajectory within a spatio-temporal lattice $\mathcal{L} \times \mathcal{T}$, where $\mathcal{L} = \{1, \dots, L\}$ denotes spatial token indices and $\mathcal{T} = [0, 1]$ represents diffusion time.
To accelerate this process and reduce variance, we propose \textbf{Spatio-Temporal Pruning (STP)}. 
STP unifies two pruning strategies corresponding to the two dimensions of the lattice: \textit{Spatial Pruning} (reducing the effective action space) and \textit{Temporal Pruning} (reducing the integration steps).
\subsection{Spatial Pruning}
\label{subsec:spatial_pruning}
Spatial pruning constrains the search space to prevent inefficient exploration. This is theoretically motivated by the \textit{Semi-Offline RL} paradigm~\cite{chen2023semi}, which theoretically demonstrates that mixing offline data with online generation can balance exploration efficiency and generation quality. Compared to AR models, this paradigm is particularly well-suited for dLLMs due to their inherent ability to handle internal masks.

We reformulate the denoising process in dLLMs from a single trajectory into a two-stage procedure, parameterized by the spatial pruning ratio $\gamma$ . Unlike the conventional schedule where timestep $t$ decreases uniformly from 1 to 0 in steps of $\Delta t$, the transition between our two stages is determined by $( (1 - \gamma)$.

\textbf{Stage 1} (Spatially-pruned Initialization, $1\geq t \geq \gamma$):
\begin{equation}
\label{sp}
    x_{t}^i = \begin{cases}\texttt{[MASK]}, & \text{if } e_i = 1 \\ x_\text{fixed}^i, & \text{if } e_i=0\end{cases}
\end{equation}
where $e_i \sim \text{Bernoulli}(\gamma)$ controls whether exploration is performed for the 
$i$-th token, and $x_{\text{fixed}}^i$ denotes an anchor sequence from a static dataset, which may come from reference solutions or prior model generations.

\textbf{Stage 2} (Spatially-pruned Denoising, $\gamma > t \geq 0$):
\begin{equation}
    x_t^i = \begin{cases}
    \hat{x}_t^i, & \text{if } e_i=1 \text{ and } m_i=0 \\ 
    \texttt{[MASK]}, & \text{if } e_i=1 \text{ and } m_i=1  \\
    x_\text{fixed}^i, & \text{if } e_i=0
    \end{cases}
\end{equation}
where $\hat{x}_t^i \sim \pi_\theta(\cdot|x_{t+\Delta t})$, $m_i = \mathbb{I}(i\in S_{maxconf})$, and $S_{maxconf} = argmax (\pi_\theta(\cdot|x_{t+\Delta t}))$ corresponds to the low-confidence decoding strategy~\cite{nie2025large}.
In the corresponding forward process used by the ELBO estimator,
the conditional distribution for each token at time $t$ is defined as:
\begin{equation}
    q_\text{SP}(x_t^i| x_0^i) = \begin{cases}
    1-t, & \text{if }i \notin S_\text{fixed} \text{ and } x_t^i=x_0^i \\ 
    t, & \text{if }i \notin S_\text{fixed} \text{ and } x_t^i=\texttt{[MASK]} \\
    1, & \text{if }i \in S_\text{fixed} \text{ and } x_t^i=x_\text{fixed}^i
    \end{cases}.
\end{equation}
where $S_{\text{fixed}}=\{i|e_i=0\}$. Thus $|S_{\text{fixed}}| =\gamma L$.

By fixing these tokens, we reduce the space of the policy $\pi_\theta$ from $\mathcal{V}^L$ to $\mathcal{V}^{L - |S_{\text{fixed}}|}$ and shorten the sequence length for ELBO estimation from $L$ to $L-|S_{\text{fixed}}|$. This accelerates the sampling process and \textbf{lowers the variance} of the estimator as derived in Theorem~\ref{thm:efficiency} and Theorem~\ref{thm:variance_reduction}.
\subsection{Temporal Pruning}
\label{subsec:temporal_pruning}
While spatial pruning operates on the spatial dimension, temporal pruning operates on the temporal dimension.
Standard diffusion is inherently a \textit{fully observable} Markovian process: the transition to state $x_{t-\Delta t}$ requires explicit observation of the immediate predecessor $x_t$. This necessitates a rigorous, step-by-step diffusion until $t=0$.

We propose accelerating late-stage denoising via a \textit{partially observable} process~\cite{chen2023semi} since the later steps are primarily dedicated to local syntactic refinement~\cite{huang2025reinforcing}. 
The intermediate states between a cutoff time $t_{\text{cutoff}}$ and $0$ are treated as unobserved variables, allowing the model to bypass the iterative chain and directly predict the final outcome.

Specially, we decompose the \textbf{Stage 2} in Section~\ref{subsec:spatial_pruning} into two sub-stages, controlled by a temporal cutoff 
$t_{\text{cutoff}} > 0$.

\textbf{Stage 2a} (Spatially-pruned Denoising, $\gamma > t \geq \tcut$):
\begin{equation}
    x_t^i = \begin{cases}
    \hat{x}_t^i, & \text{if } e_i=1 \text{ and } m_i=0 \\ 
    \texttt{[MASK]}, & \text{if } e_i=1 \text{ and } m_i=1  \\
    x_\text{fixed}^i, & \text{if } e_i=0
    \end{cases}
\end{equation}
is equivalent to the undivided Stage 2.

\textbf{Stage 2b} (Temporally-pruned Denoising, $\tcut > t \geq 0$):
\begin{equation}
    x_t^i = \begin{cases}
    \tilde{x}_{t}^i, & \text{if } e_i=1 \\ 
    x^i_\text{fixed}, & \text{if } e_i=0  \\
    \end{cases}
\end{equation}
where $\tilde{x}_{t}^i = argmax (\pi_\theta(\cdot|x_{\tcut}))$.

This design leverages the inherent advantage of the masked dLLM architecture, which enables the simultaneous decoding of tokens at all positions in one single forward pass.
By pruning the redundant steps $t \in (t_{\text{cut}}, 0]$, we reduce the computational cost of the trajectory sampling from $T$ steps to $T \times (1-t_{\text{cut}})$ steps.

\section{Theoretical Analysis}
\label{sec:theory}

In this section, we provide a theoretical grounding for Spatio-Temporal Pruning (STP). We analyze its impact on sampling efficiency and the variance of the Evidence Lower Bound (ELBO) estimator. Furthermore, we explicitly connect these properties to the stability of GRPO training.

\subsection{Sampling Efficiency}
The standard sampling process for a dLLM requires $N$ iterative denoising steps to transition from $t=1$ to $t=0$. Let $\mathcal{C}_{\text{step}}$ denote the computational cost of a single forward pass (which is roughly constant).
\begin{theorem}
    \label{thm:efficiency}
    (Computational Complexity Reduction) Let $N$ be the total number of diffusion steps, $\gamma$ be the spatial pruning ratio, and $\tcut \in (0, 1)$ be the temporal pruning cutoff.
    The standard sampling cost is $\mathcal{C}_{\text{std}} = N \cdot \mathcal{C}_{\text{step}}$. Under STP, the sampling cost is reduced to:
    \begin{equation}
        \mathcal{C}_{\text{STP}} \approx \left(  (1 - \gamma - \tcut) \cdot N + 1 \right) \cdot \mathcal{C}_{\text{step}}.
    \end{equation}
\end{theorem}
This linear reduction in temporal complexity directly translates to wall-clock speedup, as sampling time is dominated by the number of forward passes.

\subsection{Variance Reduction of the ELBO Estimator}
A core contribution of STP is the stabilization of the RL training signal. The ELBO estimator $\hat{\mathcal{B}}_{\pi_\theta}(y)$ used in optimization is stochastic due to the sampling of time steps $t$ and the noise in the diffusion process.

\begin{assumption}
    \label{assumption:elbo}
(Boundedness of ELBO Estimator) We assume the ELBO estimator difference $\Delta \hat{\mathcal{B}} = \hat{\mathcal{B}}_{\pi_\theta}(y) - \hat{\mathcal{B}}_{\pi_{\text{old}}}(y)$ is bounded. Specifically, there exists a constant $K > 0$ such that $|\Delta \hat{\mathcal{B}}| \leq K$.
\end{assumption}
This assumption is reasonable because $\pi_\text{old}$ and $\pi_\theta$ are separated by only a few training steps, and constraints like clipping and KL regularization explicitly ensure their closeness.
Based on this assumption, we analyze the variance reduction of ELBO estimator via STP.
\begin{theorem}
    \label{thm:variance_reduction}
    (Variance Reduction in ELBO) Let $\hat{\mathcal{B}}^{STP}_{\pi_\theta}(y)$ denote the ELBO estimator computed under the STP framework. The variance of the estimator is strictly bounded by the variance of the standard estimator:
    \begin{equation}
        \mathbb{V}[\hat{\mathcal{B}}^{STP}_{\pi_\theta}(y)] < \mathbb{V}[\hat{\mathcal{B}}^{standard}_{\pi_\theta}(y)].
    \end{equation}
\end{theorem}
This theorem guarantees that STP reduces the variance of ELBO estimation, thereby stabilizing RL training.
The proof is provided in Appendix~\ref{app:proofs}.

\subsection{Impact on GRPO Stability}
GRPO relies on importance sampling to estimate policy gradients off-policy. 
In Masked dLLMs, the likelihood ratio is approximated via ELBOs: $\rho_i \approx \exp({\mathcal{B}}_{\pi_\theta}(o_i) - {\mathcal{B}}_{\pi_{\theta_{old}}}(o_i))$.
Let $\hat{\mathcal{L}}_{\text{GRPO-E}}$ be the empirical GRPO loss using ELBO estimates. Let  $\Delta \hat{\mathcal{B}} = \hat{\mathcal{B}}_{\pi_\theta} - \hat{\mathcal{B}}_{\pi_{\text{old}}}$  be the random variable representing the estimator difference, and $\Delta \mathcal{B}$ be the true ELBO difference. 
\begin{theorem}
    \label{theorem:grpo_impact}
    (Bias and Variance Reduction in GRPO) Under Assumption ~\ref{assumption:elbo}, the bias and variance induced by the ELBO estimation are strictly bounded by the variance of the ELBO difference estimator $\mathbb{V}[\Delta \hat{\mathcal{B}}]$:

\begin{equation}
\begin{split}
    \left| \mathbb{E}[\hat{\mathcal{L}}_{\text{GRPO-E}}] - \mathcal{L}_{\text{GRPO-E}} \right| \leq C_1 \cdot \mathbb{E}_{data}[\mathbb{V}[\Delta \hat{\mathcal{B}}]] + \\
    C_2 \cdot \mathbb{E}_{data}[\sqrt{\mathbb{V}[\Delta \hat{\mathcal{B}}]}]
\end{split}    
\end{equation}

$$\mathbb{V}[\hat{\mathcal{L}}_{\text{GRPO-E}}] \leq V_{\text{data}} + C_3 \cdot \mathbb{E}_{data}[\mathbb{V}[\Delta \hat{\mathcal{B}}]]$$

where $C_1, C_2, C_3$ are constants depending on the advantage $A$ and the bound $K$.
\end{theorem}
The full proof is provided in Appendix~\ref{proof2}.
This confirms that reducing the variance of the ELBO estimator ($\mathbb{V}[\Delta \hat{\mathcal{B}}]$) directly tightens the bounds on both bias and variance of the GRPO objective.

High variance in ELBO estimation leads to an overestimation of the importance weights, potentially causing exploding gradients and unstable updates in GRPO. By reducing $\mathbb{V}[\hat{\mathcal{B}}]$ via Theorem \ref{thm:variance_reduction}, STP minimizes this bias and the variance of the gradient estimator, thereby enabling more stable and efficient reinforcement training.

\section{Experiments}
\label{sec:experiments}
\begin{table*}[t!]
  \caption{Main results on reasoning benchmarks. We report the test accuracy (Acc., \%) and total training wall-clock time (Time, s). The base model is LLaDA-8B-Instruct. \textbf{Bold} indicates the best performance. The last row shows the relative accuracy gain and relative time reduction of our method compared to the best baseline.}
  \label{tab:main_results}
  \centering
  \begin{small}
  \begin{sc}
  \begin{tabular}{lcccccc}
  \toprule
  \multicolumn{1}{c}{\multirow{2}{*}{Methods}} & \multicolumn{2}{c}{MATH} & \multicolumn{2}{c}{GSM8K} & \multicolumn{2}{c}{Countdown} \\ 
  \cmidrule(lr){2-3} \cmidrule(lr){4-5} \cmidrule(lr){6-7}
   & Acc. & Time (s) & Acc. & Time (s) & Acc. & Time (s) \\ 
  \midrule
  LLaDA-8B-Instruct & 32.80\% & - & 77.79\% & - & 16.80\% & - \\
  GRPO w/ ELBO & 34.20\% & 126,045 & 80.52\% & 163,035 & 36.33\% & 164,657 \\
  Diffu-GRPO & 32.80\% & 113,358 & 80.21\% & 146,641 & 25.39\% & 147,100 \\ 
  \midrule
  \textbf{STP (Ours)} & \textbf{36.20\%} & \textbf{98,537} & \textbf{80.97\%} & \textbf{144,536} & \textbf{66.02\%} & \textbf{142,055} \\ 
  \textit{Improv. (vs. Best Baseline)} & \textcolor{teal}{+5.85\%} & \textcolor{teal}{-13.07\%} & \textcolor{teal}{+0.56\%} & \textcolor{teal}{-1.44\%} & \textcolor{teal}{+81.72\%} & \textcolor{teal}{-3.43\%} \\
  \bottomrule
  \end{tabular}
  \end{sc}
  \end{small}
\end{table*}
In this section, we present a series of comprehensive experiments designed to validate the effectiveness of Spatio-Temporal Pruning (STP)
\subsection{Experimental Setup}
\label{subsec:setup}
We conduct our experiments based on {LLaDA-8B-Instruct}~\cite{nie2025large} model. To comprehensively assess reasoning capabilities, we evaluate performance on four challenging benchmarks: GSM8K~\cite{cobbe2021training} and MATH~\cite{hendrycks2021measuring} for mathematical reasoning, Countdown~\cite{tinyzero} for logical reasoning.

We compare our proposed method against the following representative reinforcement learning baselines for dLLMs: \begin{itemize}[nosep,leftmargin=1em,labelwidth=*,align=left]
\item \textbf{GRPO w/ ELBO}: A direct adaptation of GRPO~\cite{shao2024deepseekmath} for dLLMs. It employs the standard Evidence Lower Bound (ELBO) estimator as a surrogate for the intractable likelihood to compute policy gradients. \item \textbf{Diffu-GRPO}~\cite{zhao2025d1}: An efficiency-oriented algorithm that approximates the likelihood via a one-step unmasking strategy, reducing the computational overhead compared to ELBO-based estimation. \end{itemize}
\subsection{Implementation Details}
\paragraph{Training Configuration.}We conduct all experiments on NVIDIA A800 GPUs (80GB). Following the baselines, we employ Low-Rank Adaptation (LoRA) for parameter-efficient fine-tuning with rank $r = 128$ and scaling factor $\alpha = 64$. We optimize the model using AdamW with a learning rate of $3 \times 10^{-6}$.
Detailed hyperparameters are listed in Table~\ref{tab:hyperparams} in Appendix~\ref{app:exp_details}.
\paragraph{STP Configuration.}For our Spatio-Temporal Pruning (STP), we utilize the model's own pre-generations as the source for the static dataset. To prevent reward hacking, tokens enclosed within specific answer delimiters (e.g., \texttt{<answer>}... \texttt{</answer>}) are excluded from the fixed set. Unless otherwise stated, we set the spatial pruning ratio to $\gamma = 0.05$ (representing the proportion of fixed tokens) and the temporal pruning cutoff to $t_{\text{cutoff}} = 0.05$.
\paragraph{Optimization and Evaluation.}Our training objective adopts the reinforcement learning formulation from Difu-GRPO~\cite{zhao2025d1}, incorporating KL divergence penalty and gradient clipping for stability. We do not apply prompt masking during training. To ensure consistency, the generation sequence length is fixed to $L=256$ for both training and evaluation. We report performance on the final checkpoint.
\subsection{Main Results}
\label{subsec:main_results}
We report the GRPO training accuracy, total training wall-clock time (in seconds), and relative importvement with the best baseline (imporv.) on the evaluated reasoning benchmarks in Table~\ref{tab:main_results}. The results demonstrate that STP achieves a superior trade-off, surpassing the state-of-the-art in both dimensions simultaneously.


\textbf{Superior Computational Efficiency.} As shown in the bottom row of Table~\ref{tab:main_results}, STP consistently reduces training latency, even when compared to the efficiency-oriented \textit{Diffu-GRPO} baseline. We observe relative time reductions of \textbf{13.07\%} on MATH and \textbf{3.43\%} on Countdown. This confirms that STP's pruning strategy effectively eliminates redundant computations, making it the fastest training method among all compared approaches while maintaining full optimization capability.

\textbf{State-of-the-Art Effectiveness.} In terms of reasoning accuracy, STP outperforms the strong \textit{GRPO w/ ELBO} baseline across all tasks.On the {MATH} benchmark, STP achieves an accuracy of \textbf{36.20\%}, exceeding the best baseline by \textbf{2.0\%} in absolute terms. On {GSM8K}, our method maintains a steady advantage, pushing the accuracy to \textbf{80.97\%}. Most notably, on the {Countdown} task, STP achieves a transformative relative improvement of \textbf{81.72\%} (from 36.33\% to 66.02\%). This confirms that STP gains efficiency without compromising training effectiveness.
\subsection{Analysis of ELBO Variance Reduction}
\label{subsec:variance_analysis}
We theoretically analyzed that STP reduces the variance of the ELBO estimator, thereby stabilizing the policy gradient in Theorem~\ref{thm:variance_reduction} and Theorem~\ref{theorem:grpo_impact}. To empirically verify this, we tracked the estimation variance throughout the training process on the MATH dataset. We record 1000 training steps ELBO estimator values for each sampled trajectory and calculate the variance. We compare the variance distribution of the standard \textit{GRPO w/ ELBO} against \textit{STP}.
\begin{figure*}[ht]
    \centering
    \begin{subfigure}[b]{0.96\columnwidth}
        \centering
        \includegraphics[width=\linewidth]{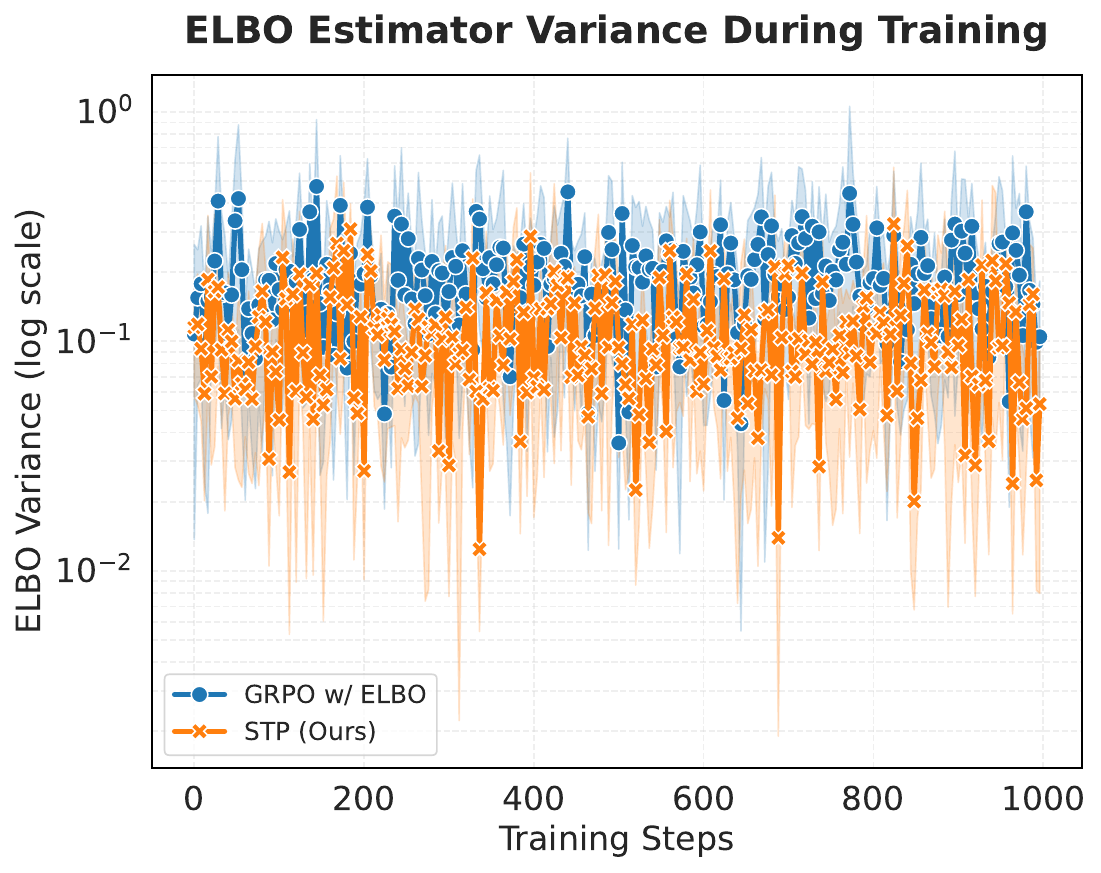}
        \caption{Variance over Training Steps}
        \label{fig:var_steps}
    \end{subfigure}
    \hfill
    \begin{subfigure}[b]{0.96\columnwidth}
        \centering
        \includegraphics[width=\linewidth]{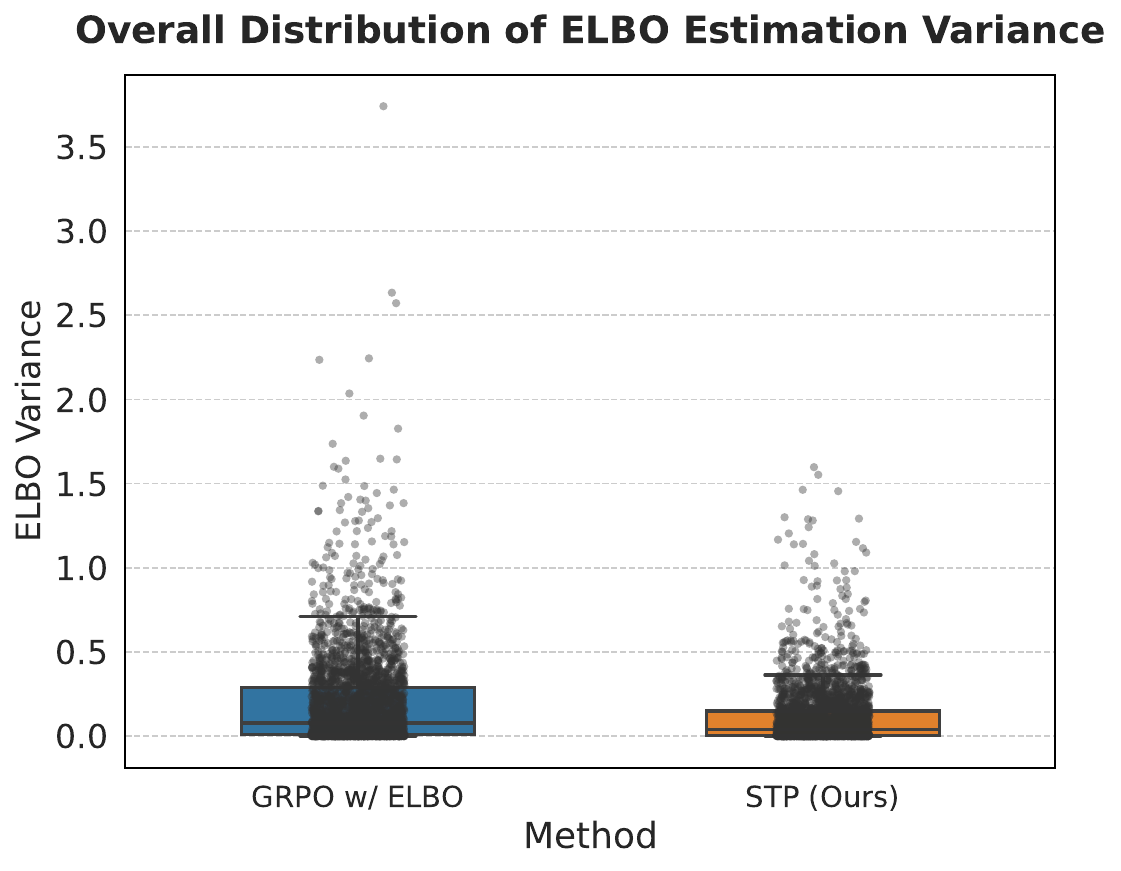}
        \caption{Variance Distribution}
        \label{fig:var_dist}
    \end{subfigure}
    \caption{\textbf{Empirical Verification of Variance Reduction.} (a) The ELBO estimation variance recorded during training dynamics. STP (Orange) consistently exhibits lower variance compared to the standard method GRPO w/ ELBO (Blue). (b) The aggregate distribution of variance shows that STP significantly lowers the median variance and suppresses extreme outliers, validating our theoretical bounds in Theorem~\ref{thm:variance_reduction}.}
    \label{fig:variance_analysis}
\end{figure*}
Figure~\ref{fig:variance_analysis} presents the results.

As shown in Fig.~\ref{fig:var_steps}, the standard ELBO estimator exhibits high and fluctuating variance throughout training, whereas STP maintains consistently low variance. This stability is crucial for importance-sampling-based methods like GRPO, as high-variance likelihood estimates can cause unstable gradient updates.

Furthermore, the boxplot in Fig.~\ref{fig:var_dist} confirms that the variance reduction achieved by STP is statistically significant. STP not only lowers the median variance but also substantially narrows the interquartile range and reduces outliers.

These empirical results are consistent with our theoretical analysis in Section~\ref{sec:theory}, and corroborate the performance improvements shown in Section~\ref{subsec:main_results}. This confirms that STP reduces the variance of the ELBO estimator, enabling more stable and efficient reinforcement learning.
\subsection{Compatibility with Other RL Algorithms}
\label{subsec:compatibility}

Since STP optimizes the fundamental trajectory sampling process on the spatio-temporal dimension, it functions as an orthogonal acceleration module compatible with other enhanced RL algorithms on dLLMs. To verify this, we integrated STP on the MATH dataset with Diffu-GRPO~\cite{zhao2025d1} and SPG~\cite{wang2025spg}. SPG mitigates gradient bias by leveraging both an upper and a lower bound of the true log-likelihood.

As shown in Table~\ref{tab:orthogonality}, STP consistently boosts both efficiency and performance. Notably, integrating STP with Diffu-GRPO yields a significant accuracy improvement (+1.6\%) and a $\sim$30\% speedup, suggesting that the high-quality priors from Spatial Pruning help correct the optimization direction. For {SPG}, STP further pushes the performance boundary to {36.60\%} while reducing training time by {7.5\%}. 
This confirms that STP is compatible with other enhanced RL algorithms on dLLMs, providing complementary benefits by reducing the sampling overhead and ELBO variance.
\begin{table}[t]
\caption{Compatibility of STP with advanced RL algorithms on MATH. STP consistently improves both accuracy and training speed when integrated with Diffu-GRPO and SPG.}
\label{tab:orthogonality}
\begin{center}
\begin{small}
\begin{sc}
\begin{tabular}{lccc}
\toprule
Method & Accuracy & Time (s) & Speedup \\
\midrule
Diffu-GRPO & 32.80\% & 113,358 & - \\
\textbf{+ STP} & \textbf{34.40\%} & \textbf{79,529} & \textbf{+29.8\%} \\
\midrule
SPG & 36.10\% & 89,525 & - \\
\textbf{+ STP} & \textbf{36.60\%} & \textbf{82,830} & \textbf{+7.5\%} \\
\bottomrule
\end{tabular}
\end{sc}
\end{small}
\end{center}
\end{table}
\subsection{STP as Guided Exploration: Leveraging Static Solutions}
\label{subsec:guided_exploration}
Beyond efficiency, STP mitigates the "cold start" exploration challenge in challenging tasks by injecting high-quality priors. Unlike standard RL which decodes from scratch, STP utilizes \textit{Spatial Pruning} to fix a subset of tokens based on static data, effectively acting as guided exploration akin to Supervised Fine-Tuning (SFT).

To validate this, we compare STP using self-generated priors versus ground-truth (GT) solutions in Table~\ref{tab:sft_benefit}. While STP with model generations already surpasses the baseline (+2.0\%) due to variance reduction, using GT solutions further boosts accuracy to 36.60\%. This demonstrates that STP successfully leverages external signals to navigate the vast solution space of difficult reasoning tasks, providing SFT-like benefits with negligible computational overhead (99k vs. 98k seconds).
\begin{table}[t]
\caption{Comparison of exploration sources on the MATH dataset. Using ground truth (GT) solutions yields the highest accuracy, demonstrating the SFT-like benefit of STP in guiding exploration for difficult tasks.}
\label{tab:sft_benefit}
\begin{center}
\begin{small}
\begin{sc}
\begin{tabular}{lccc}
\toprule
\multicolumn{1}{c}{Method} & Source & Acc. & Time (s) \\
\midrule
GRPO w/ ELBO & - & 34.20\% & 126,045 \\
STP (Ours) & Model Gens & 36.20\% & \textbf{98,537} \\
STP (Ours) & GT Solutions & \textbf{36.60\%} & 99,756 \\
\bottomrule
\end{tabular}
\end{sc}
\end{small}
\end{center}
\vskip -0.1in
\end{table}

\subsection{Ablation Study}
\label{subsec:ablation}
To rigorously understand the individual contributions of our proposed components and the impact of hyperparameter choices, we conduct a series of ablation studies on the MATH dataset.

\subsubsection{Impact of Pruning Components}
We isolate the effects of Spatial Pruning (SP) and Temporal Pruning (TP) to evaluate their distinct roles in the trade-off between efficiency and performance. Table~\ref{tab:ablation_components} presents the results. Applying Spatial Pruning alone improves accuracy over the baseline ($34.20\% \to 35.20\%$) and reduces training time by approximately 18\%, confirming that anchoring the generation with high-confidence static tokens stabilizes the exploration. Temporal Pruning alone yields the highest accuracy ($36.60\%$), suggesting that the early-exit mechanism effectively filters out the noisy tail of the diffusion process without compromising the quality of the generation. The combined STP framework achieves the lowest training time ($98,537$s, a $\sim$21.8\% speedup) while maintaining competitive accuracy ($36.20\%$), demonstrating that the two strategies are complementary: SP reduces the spatially  inefficient exploration, while TP reduces redundancy on temporal dimension.

\begin{table}[t]
\caption{Ablation of STP components on MATH. TP maximizes accuracy, while the combined STP offers the best efficiency-performance trade-off.}
\label{tab:ablation_components}
\begin{center}
\begin{small}
\begin{sc}
\begin{tabular}{lcccc}
\toprule
Method & SP & TP & Acc. & Time (s) \\
\midrule
GRPO w/ ELBO & $\times$ & $\times$ & 34.20\% & 126,045 \\
Spatial Only & $\checkmark$ & $\times$ & 35.20\% & 102,553 \\
Temporal Only & $\times$ & $\checkmark$ & \textbf{36.60\%} & 102,190 \\
{STP (Ours)} & $\checkmark$ & $\checkmark$ & {36.20\%} & \textbf{98,537} \\
\bottomrule
\end{tabular}
\end{sc}
\end{small}
\end{center}
\vskip -0.1in
\end{table}

\subsubsection{Hyperparameter Sensitivity}
We investigate the sensitivity of STP to two key hyperparameters: the \textit{Spatial pruning ratio} ($\gamma$) and the \textit{Temporal pruning cutoff} ($\tcut$). We compare aggressive pruning settings ($0.1$) against conservative ones ($0.05$). As shown in Table~\ref{tab:ablation_hyperparams}, setting $\gamma=0.1$ degrades performance to baseline levels ($34.20\%$), likely because an overly aggressive fixing strategy hurts the diversity of the generated sequences. Decreasing $\gamma$ to $0.05$ recovers accuracy to $35.20\%$. Similarly, for Temporal Pruning, a lower cutoff $\tcut=0.05$ yields better accuracy than $\tcut=0.1$. These results indicate that a conservative pruning strategy strikes a better balance, ensuring sufficient exploration and generation quality. We therefore adopt $0.05$ for both parameters in our main experiments.\looseness=-1

\begin{table}[h]
\caption{Sensitivity analysis of hyperparameters $\gamma$ and $\tcut$ on MATH. Conservative pruning ($0.05$) consistently outperforms aggressive pruning($0.1$).}
\label{tab:ablation_hyperparams}
\begin{center}
\begin{small}
\begin{sc}
\begin{tabular}{cccc}
\toprule
$\gamma$ & $\tcut$ & Acc. & Time (s) \\
\midrule
0.10 & 0.00 & 34.20\% & 100,935 \\
0.05 & 0.00 & 35.20\% & 102,553 \\
0.00 & 0.10 & 35.20\% & 102,165 \\
0.00 & 0.05 & \textbf{36.60\%} & 102,190 \\
0.05 & 0.05 & 36.20\% & \textbf{98,537} \\
\bottomrule
\end{tabular}
\end{sc}
\end{small}
\end{center}
\vskip -0.1in
\end{table}

\subsubsection{Spatial Pruning Strategy}
Finally, we examine the token selection strategy for Spatial Pruning by comparing \textit{Random Selection} against \textit{Low Confidence Retention}. In the latter strategy, we use tokens with high-confidence as fixed tokens. Table~\ref{tab:ablation_strategy} shows that Low Confidence Retention significantly outperforms Random Selection ($35.20\%$ vs. $34.60\%$). The superiority of this strategy stems from its ability to intelligently allocate the exploration budget: by fixing the ``easy" tokens that the model has already mastered, it forces the model to focus its exploration on positions with higher uncertainty, thereby optimizing the learning efficiency for complex reasoning steps.\looseness=-1

\begin{table}[h]
\caption{Comparison of Spatial Pruning strategies ($\gamma=0.05$). Retaining low-confidence tokens for exploration yields superior results.}
\label{tab:ablation_strategy}
\begin{center}
\begin{small}
\begin{sc}
\begin{tabular}{lcc}
\toprule
Strategy & Acc. & Time (s) \\
\midrule
Random & 34.60\% & \textbf{98,234} \\
Low Confidence (Ours) & \textbf{35.20\%} & 102,553 \\
\bottomrule
\end{tabular}
\end{sc}
\end{small}
\end{center}
\vskip -0.1in
\end{table}
\section{Related Work}

\paragraph{Diffusion Models for Text Generation.} Diffusion models have recently expanded from continuous domains to discrete text generation. 
Early approaches modeled text in continuous embedding spaces~\cite{li2022diffusion,gong2022diffuseq}.
Recently, Masked Diffusion Models (MDMs) that operate directly on discrete tokens
such as LLaDA~\cite{nie2025large} and Dream~\cite{ye2025dream}, demonstrate generation capabilities competitive with autoregressive models with superior efficiency. 

\paragraph{Reinforcement Learning for Diffusion Language Models.} While Reinforcement Learning (RL) algorithms like PPO and GRPO are standard for enhancing autoregressive models~\cite{schulman2017proximal,shao2024deepseekmath}, applying them to dLLMs is non-trivial due to the intractability of the exact likelihood. Recent advancements focus on adapting GRPO by approximating likelihoods concentrating on large time comsumption and high variance of the estimation: both DiffuGRPO~\cite{zhao2025d1} and UniGRPO~\cite{yang2025mmada} employ a one-step unmasking approximation to improve the efficiency of likelihoods estimation. To mitigate the high variance of the estimation, methods like wd1~\cite{tang2025wd1} and DiffuCoder~\cite{gong2025diffucoder} propose weighted objectives or coupled sampling schemes. 
Unlike these approaches which primarily focus on estimator formulation, our STP constrains the sampling trajectory in both spatial and temporal dimensions, reducing the variance of the estimation fundamentally while simultaneously lowering the computational cost.
\section{Conclusion and Future Work}
We propose STP, a principled framework that tackles the critical bottlenecks of efficiency and stability in Reinforcement Learning for dLLMs. By effectively identifying and pruning redundancy across both spatial and temporal dimensions, STP significantly reduces the computational overhead of trajectory sampling. More importantly, we theoretically and empirically demonstrated that STP lowers the variance of the ELBO estimator, providing a high-quality training signal that is essential for stable policy optimization. Our results on reasoning tasks confirm that STP not only accelerates training but also serves as a guided exploration mechanism, enabling the model to achieve state-of-the-art performance. Furthermore, STP is orthogonal to other algorithmic advancements, making it a versatile plug-and-play module for future dLLM research. Future work includes exploring adaptive scheduling strategies that dynamically adjust pruning based on instance difficulty and extending this paradigm to multimodal diffusion settings.
\section*{Impact Statements}
This paper presents work whose goal is to advance the field of machine learning. There are many potential societal consequences of our work, none of which we feel must be specifically highlighted here.
\bibliography{example_paper}
\bibliographystyle{icml2026}

\newpage
\appendix
\onecolumn

\section{Proofs of Theoretical Analysis}
\subsection{Proof of Theorem \ref{thm:variance_reduction}}
\label{app:proofs}
\begin{proof}
    
We analyze the variance of the ELBO estimator. Recall the standard estimator definition:
\begin{equation}
    \hat{\mathcal{B}}(y) = \frac{1}{K} \sum_{k=1}^K \ell(\theta, y, t_k, \epsilon_k)
\end{equation}
where $t_k \sim \mathcal{U}[0,1]$ and $\ell$ is the reweighted loss at step $t$. By the Law of Total Variance, we can decompose the variance into two terms: variance due to time step sampling ($\mathcal{V}_t$) and variance due to noise sampling conditioned on time ($\mathcal{V}_{noise}$):
\begin{equation}
    \mathbb{V}[\hat{\mathcal{B}}] = \mathbb{V}_t[\mathbb{E}_{\epsilon}[\ell|t]] + \mathbb{E}_t[\mathbb{V}_{\epsilon}[\ell|t]].
\end{equation}

\textbf{1. Effect of Spatial Pruning:}
In STP, we fix a set of tokens $S_{fixed}$. The loss function becomes a sum over only the non-fixed tokens $i \notin S_{fixed}$. Let $L$ be the sequence length and $L' = L - |S_{fixed}|$. Assuming token-wise losses have bounded covariance, the variance of the sum scales with the number of terms.
Let $\mathcal{L}_{full} = \sum_{i=1}^L \ell_i$ and $\mathcal{L}_{pruned} = \sum_{i \notin S_{fixed}} \ell_i$.
\begin{equation}
    \mathbb{V}_{\epsilon}[\mathcal{L}_{pruned}|t] \approx \frac{L'}{L} \mathbb{V}_{\epsilon}[\mathcal{L}_{full}|t] < \mathbb{V}_{\epsilon}[\mathcal{L}_{full}|t].
\end{equation}
Thus, spatial pruning directly reduces the second term $\mathbb{E}_t[\mathbb{V}_{\epsilon}[\ell|t]]$.

\textbf{2. Effect of Temporal Pruning:}
Tokens generated via temporal pruning are still treated as "generated" tokens. During the ELBO calculation for the RL update, these tokens follow the standard probabilistic formulation. Therefore, the decrease in variance strictly originates from Spatial Pruning.

Combining both effects, $\mathbb{V}[\hat{\mathcal{B}}^{STP}] < \mathbb{V}[\hat{\mathcal{B}}^{standard}]$.
\end{proof}

\subsection{Proof of Theorem \ref{theorem:grpo_impact}}
\label{proof2}
\begin{lemma}
\label{lemma1}
    (Bias and Variance of Exponential Transformation) Let $X$ be a random variable satisfying $|X| \le K$ with mean $\mu = \mathbb{E}[X]$ and variance $\mathbb{V}[X]$. Let $f(x) = e^x$. Then the transformed variable $Y = f(X)$ satisfies:
\begin{equation}
\left| \mathbb{E}[f(X)] - f(\mu) \right| \leq \frac{C_{\exp}}{2} \mathbb{V}[X]
\end{equation}
\begin{equation}
    \mathbb{V}[f(X)] \leq C_{\exp}^2 \mathbb{V}[X]
\end{equation}
where $C_{\exp} = e^K$ is the Lipschitz constant of $e^x$ on $[-K, K]$.
\end{lemma}
\begin{proof}
Consider the second-order Taylor expansion of $f(X)$ around $\mu$:
$$f(X) = f(\mu) + f'(\mu)(X - \mu) + \frac{1}{2}f''(\xi)(X - \mu)^2$$
where $\xi$ lies between $X$ and $\mu$. Taking the expectation on both sides:
$$\mathbb{E}[f(X)] = f(\mu) + f'(\mu)\underbrace{\mathbb{E}[X - \mu]}_{0} + \frac{1}{2}\mathbb{E}[f''(\xi)(X - \mu)^2]$$
Since $f(x) = e^x$, we have $f''(x) = e^x$. Given the boundedness assumption $|X| \le K$, implies $|\xi| \le K$, thus $|f''(\xi)| \le e^K = C_{\exp}$.
$$\left| \mathbb{E}[f(X)] - f(\mu) \right| = \left| \frac{1}{2}\mathbb{E}[f''(\xi)(X - \mu)^2] \right| \leq \frac{C_{\exp}}{2} \mathbb{E}[(X - \mu)^2] = \frac{C_{\exp}}{2} \mathbb{V}[X]$$

Since $f(x) = e^x$ is continuously differentiable and bounded on $[-K, K]$, it is Lipschitz continuous with constant $L = \sup_{z \in [-K, K]} |f'(z)| = e^K = C_{\exp}$.
Using the standard property of Lipschitz functions on variance:
$$\mathbb{V}[f(X)] = \mathbb{E}[(f(X) - \mathbb{E}[f(X)])^2] \leq \mathbb{E}[(f(X) - f(\mu))^2]$$
Since $f$ is $C_{\exp}$-Lipschitz:
$$|f(X) - f(\mu)| \leq C_{\exp} |X - \mu|$$
Squaring both sides and taking expectations:
$$\mathbb{E}[(f(X) - f(\mu))^2] \leq C_{\exp}^2 \mathbb{E}[(X - \mu)^2] = C_{\exp}^2 \mathbb{V}[X]$$
Thus, $\mathbb{V}[f(X)] \leq C_{\exp}^2 \mathbb{V}[X]$.
\end{proof}

\begin{lemma}
\label{lemma2}
    (Lipschitz Continuity of GRPO Objective) Let $r$ be the importance ratio and $A$ be the advantage. The per-sample GRPO objective function $g(r; A) = \min(rA, \text{clip}(r, 1-\epsilon, 1+\epsilon)A)$ is Lipschitz continuous with respect to $r$. Specifically, for a fixed advantage $A$, there exists a constant $C_{\text{clip}} = |A|$ such that:
    \begin{equation}
        |g(r_1; A) - g(r_2; A)| \leq C_{\text{clip}} |r_1 - r_2|
    \end{equation}
\end{lemma}
\begin{proof}
The function $g(r; A)$ is a composition of linear scaling, clipping, and the minimum operator.
\begin{itemize}
    \item The function $h_1(r) = rA$ is Lipschitz with constant $|A|$.
    \item he function $h_2(r) = \text{clip}(r, 1-\epsilon, 1+\epsilon)A$ is Lipschitz with constant 0 (in clipped regions) or $|A|$ (in unclipped regions), hence bounded by $|A|$.
    \item The minimum of two Lipschitz functions is Lipschitz with a constant equal to the maximum of their Lipschitz constants.
\end{itemize}
Thus, $g(r; A)$ is $|A|$-Lipschitz continuous w.r.t $r$. 
\end{proof}

\begin{proof}
    (Theorem \ref{theorem:grpo_impact}: Bias and Variance in GRPO Estimates)

    For a single sample tuple $(x, y)$ with advantage $A$, let $\hat{r} = e^{\Delta \hat{\mathcal{B}}}$ and $r^* = e^{\Delta \mathcal{B}}$. The loss is $L(\hat{r}) = g(\hat{r}; A)$.
    
    \textbf{1. Bias Analysis:} We aim to bound $|\mathbb{E}[L(\hat{r})] - L(r^*)|$. By triangle inequality:
$$|\mathbb{E}[L(\hat{r})] - L(r^*)| \leq \underbrace{|\mathbb{E}[L(\hat{r})] - L(\mathbb{E}[\hat{r}])|}_{\text{Jensen gap of } L} + \underbrace{|L(\mathbb{E}[\hat{r}]) - L(r^*)|}_{\text{Bias from exponential}}$$
Using Lemma~\ref{lemma2}, the first term is bounded:
$|\mathbb{E}[L(\hat{r})] - L(\mathbb{E}[\hat{r}])| \leq C_{\text{clip}} \mathbb{E}[|\hat{r} - \mathbb{E}[\hat{r}]|] \leq C_{\text{clip}} \sqrt{\mathbb{V}[\hat{r}]}$.
From Lemma~\ref{lemma1}, $\mathbb{V}[\hat{r}] \leq C_{\exp}^2 \mathbb{V}[\Delta \hat{\mathcal{B}}]$.
So, Term 1 $\leq C_{\text{clip}} C_{\exp} \sqrt{\mathbb{V}[\Delta \hat{\mathcal{B}}]}$.

For the second term, using the Lipschitz property of $L$:
$|L(\mathbb{E}[\hat{r}]) - L(r^*)| \leq C_{\text{clip}} |\mathbb{E}[\hat{r}] - r^*|$.
Recall $r^* = e^{\Delta \mathcal{B}} = e^{\mathbb{E}[\Delta \hat{\mathcal{B}}]}$ (since $\Delta \hat{\mathcal{B}}$ is unbiased~\cite{zhu2025llada}).
This is exactly the bias of the exponential function bounded in Lemma~\ref{lemma1}:
$|\mathbb{E}[e^{\Delta \hat{\mathcal{B}}}] - e^{\mathbb{E}[\Delta \hat{\mathcal{B}}]}| \leq \frac{C_{\exp}}{2} \mathbb{V}[\Delta \hat{\mathcal{B}}]$.
So, Term 2 $\leq \frac{C_{\text{clip}} C_{\exp}}{2} \mathbb{V}[\Delta \hat{\mathcal{B}}]$.

Combining both terms gives the bias bound dependent on $\mathbb{V}[\Delta \hat{\mathcal{B}}]$ and $\sqrt{\mathbb{V}[\Delta \hat{\mathcal{B}}]}$.

\textbf{2. Variance Analysis:}
Using the law of total variance:
$$\mathbb{V}[\hat{L}] = \mathbb{V}_{\text{data}}[\mathbb{E}_{\text{est}}[\hat{L}]] + \mathbb{E}_{\text{data}}[\mathbb{V}_{\text{est}}[\hat{L}]]$$
Focusing on the estimation variance term $\mathbb{V}_{\text{est}}[\hat{L}]$:
Since $L$ is $C_{\text{clip}}$-Lipschitz w.r.t $r$:
$$\mathbb{V}_{\text{est}}[L(\hat{r})] \leq C_{\text{clip}}^2 \mathbb{V}_{\text{est}}[\hat{r}]$$
From Lemma~\ref{lemma1}, $\mathbb{V}_{\text{est}}[\hat{r}] \leq C_{\exp}^2 \mathbb{V}[\Delta \hat{\mathcal{B}}]$.
Thus, the total variance is bounded by the intrinsic data variance plus a term proportional to $\mathbb{V}[\Delta \hat{\mathcal{B}}]$.
This confirms that reducing the variance of the ELBO estimator ($\mathbb{V}[\Delta \hat{\mathcal{B}}]$) directly tightens the bounds on both bias and variance of the GRPO objective.
\end{proof}

\section{Detailed Hyperparameters}
\label{app:exp_details}
Table~\ref{tab:hyperparams} lists the detailed hyperparameters used in our experiments. We largely align our settings with SPG~\cite{wang2025spg}, while adopting the KL penalty coefficients from Diffu-GRPO~\cite{zhao2025d1}.
\begin{table}[h]
\centering
\caption{Hyperparameters for Training.}
\label{tab:hyperparams}
\begin{tabular}{lc}
\toprule
\textbf{Hyperparameter} & \textbf{Value} \\
\midrule
Base Model & LLaDA-8B-Instruct \\
LoRA Rank ($r$) & 128 \\
LoRA Alpha ($\alpha$) & 64 \\
Temperature & 0.9 \\
Learning Rate & $3 \times 10^{-6}$ \\
Optimizer & AdamW \\
$\beta_1$ & 0.9 \\
$\beta_2$ & 0.99 \\
Weight Decay & 0.1 \\
Batch Size & 6 \\
Gradient Accumulation Steps & 2 \\
Inner Updates ($\mu$) & 4 \\
Clip Ratio ($\epsilon$) & 0.2 \\
KL Coefficient ($\beta$) & 0.04 \\
Prompt Masking ($p_{mask}$) & 0.0 \\
Number of Completions per Prompt & 6 \\
Number of Monte Carlo Estimation Samples & 3 \\
Sequence Length ($L$) & 256 \\
Spatial Pruning Ratio ($\gamma$) & 0.05 \\
Temporal Pruning Cutoff ($t_{\text{cutoff}}$) & 0.05 \\
Static Source & Model Pre-generations \\
\bottomrule
\end{tabular}
\end{table}

\end{document}